\documentclass[letterpaper, 10 pt, conference]{ieeeconf}  % Comment this line out
\IEEEoverridecommandlockouts                              % This command is only
\usepackage[utf8]{inputenc}
\usepackage[T1]{fontenc}

\usepackage{cite}
\usepackage[dvipsnames]{xcolor}

\usepackage{times}
\usepackage{graphicx}
\usepackage{amssymb}
\usepackage{gensymb}
\usepackage{amsmath}
\usepackage{breakurl}
\usepackage{multirow}
\usepackage{adjustbox}
\usepackage{lipsum}
\usepackage{hyperref}
\usepackage{cleveref}
\usepackage{siunitx}

\usepackage{amsthm}
\usepackage[ruled,vlined,linesnumbered]{algorithm2e}

\let\labelindent\relax
\usepackage{enumitem}
\usepackage{url}
\usepackage{booktabs}
\usepackage[ruled,vlined,linesnumbered]{algorithm2e}

\theoremstyle{proposition}
\newtheorem{proposition}{Proposition}
\theoremstyle{definition}
\newtheorem{definition}{Definition}

\title{\LARGE \bf
Detection-Aware Trajectory Generation for a Drone Cinematographer 
}

\author{Boseong Felipe Jeon, Dongsuk Shim
        and~H. Jin Kim% <-this % stops a space
\thanks{*This material is based upon work supported by the Ministry of Trade, Industry \& Energy(MOTIE, Korea) under Industrial Technology Innovation Program. No.10067206, 'Development of Disaster Response Robot System for Lifesaving and Supporting Fire Fighters at Complex Disaster Environment'}% <-this % stops a space
\thanks{Department of mechanical and aerospace engineering,
        Seoul national university of South Korea
        {\tt\small \{a4tiv,tlaehdtjr01,hjinkim\}@snu.ac.kr}}%
}

\begin{document}

\newcommand{\RRR}{{\mathbb{R}}^{3}}
\newcommand{\RRRR}{{\mathbb{R}}^{4}}
\newcommand{\RR}{{\mathbb{R}}^{2}}
\newcommand{\R}{{\mathbb{R}}}

\newcommand{\CC}{{\mathbb{C}}}
\newcommand{\pcl}{{\mathcal{P}}}

\newcommand{\Paa}{\pcl_a}
\newcommand{\Pb}{\pcl_b}
\newcommand{\transl}{{\mathbf{x}}}
\newcommand{\colorTriple}{{\mathbf{c}}}
\newcommand{\predSeq}{\{ \hat{T}_{a,i} \}_{i=n+1}^{n+N}}
\newcommand{\predSeqRef}{\{ \hat{T}_{a,i} \}_{i=1}^{N}} 

\newcommand{\pred}{\hat{T}_{a}}
\newcommand{\predI}{\hat{T}_{a,i}}

\newcommand{\viewPoint}{\transl}
\newcommand{\viewPointI}{\viewPoint_{c,i}}
\newcommand{\viewPointII}{\viewPoint_{c,i+1}}

\newcommand{\predPointI}{\viewPoint_{a,i}}
\newcommand{\viewSeq}{\{ \viewPointI \}_{i=n+1}^{n+N}}
\newcommand{\viewPath}{\sigma}
\newcommand{\viewSeqRef}{\{ \viewPointI \}_{i=1}^{N}}
\newcommand{\viewSphere}{D}

\newcommand{\imagePixel}{\mathbf{u}}
\newcommand{\imagePixelSet}{{U}}
\newcommand{\imagePixelTarget}{\mathbf{u}_a}
\newcommand{\imagePixelBack}{\mathbf{u}_b}
\newcommand{\imagePixelTargetSet}{\imagePixelSet_a}
\newcommand{\imagePixelBackSet}{\imagePixelSet_b}
\newcommand{\Bin}{V}
\newcommand{\BinTarget}{V_a}
\newcommand{\BinBack}{V_b}
\newcommand{\binMap}{\phi}

\newcommand{\aLikelihood}{h}
\newcommand{\liklihoodWeight}{w}

\newcommand{\histVal}{p}
\newcommand{\VR}{R(I_s)}

\newcommand{\Graph}{G_d}
\newcommand{\vertSet}{V}
\newcommand{\weight}{L}
\newcommand{\edgeSet}{E}
\newcommand{\distRelative}{r_d}
\newcommand{\distMax}{r_{max}}
\newcommand{\costDetectI}{L(\viewPointI|\predI)}

\newcommand{\viewPointInit}{\viewPoint_{c,0}}
\newcommand{\viewVelInit}{\dot{\viewPoint}_{c,0}}
\newcommand{\viewAccelInit}{\ddot{\viewPoint}_{c,0}}

\newcommand{\T}{\tau}
\newcommand{\yaw}{\psi}
\newcommand{\polyCoeff}{\mathbf{p}}
\newcommand{\polyOrder}{K}

\maketitle
\thispagestyle{empty}

%%%%%%%%%%%%%%%%%%%%%%%%%%%%%%%%%%%%%%%%%%%%%%%%%%%%%%%%%%%%%%%%%%%%%%%%%%%%%%%%
\begin{abstract}

This work investigates an efficient trajectory generation for chasing a dynamic target, which incorporates the detectability objective. The proposed method actively guides the motion of a cinematographer drone so that the color of a target is well-distinguished against the colors of the background in the view of the drone. 
For the objective, we define a measure of color detectability given a chasing path. After computing a discrete path optimized for the metric, we generate a dynamically feasible trajectory. The whole pipeline can be updated on-the-fly to respond to the motion of the target. For the efficient discrete path generation, we construct a directed acyclic graph (DAG) for which a topological sorting can be determined analytically without the depth-first search. The smooth path is obtained in quadratic programming (QP) framework. We validate the enhanced  performance of state-of-the-art object detection and tracking algorithms when the camera drone executes the trajectory obtained from the proposed method.          

\end{abstract}

%%%%%%%%%%%%%%%%%%%%%%%%%%%%%%%%%%%%%%%%%%%%%%%%%%%%%%%%%%%%%%%%%%%%%%%%%%%%%%%%
\section{INTRODUCTION}
% Aerial cinematography is popular
Aerial cinematography where a flying agent is employed to autonomously follow a dynamic target is an important application of a drone equipped with vision sensors. From video filming for personal usage to industrial inspection, micro aerial vehicles (MAVs) have been successfully employed, which has been spurred by developments in control \cite{neunert2016mpc}, planning \cite{ryll2019efficient}, localization\cite{qin2018vins} and object tracking \cite{lukezic2017discriminative} and detection\cite{redmon2016you}.    
% How important and challenging the object detection is for  Aerial cinematography 
In such tasks, a camera drone should detect the target (or actor) of interest and localize its state, in order to determine the motion for chasing \cite{bonatti2019toward}. Thus, the localization of target is one of the most crucial modules. However, dynamic object detection and tracking by MAVs is challenging due to  1) occlusion from obstacles, 2) motion blur from the movement of both object and the drone 3) color ambiguity with background. Addressing 1)-3) becomes more important in the cases where a drone films an actor at some distance apart due to safety.

% It is better to develop a motion strategy to overcome the difficulties 
 One option to circumvent the issues 1)-3) is to  actively utilize the motion of drone rather than entirely relying on the image-processing technique such as target detection \cite{redmon2016you} or tracking \cite{henriques2014high,lukezic2017discriminative}. 
% and there are multiple research 
Inspired by the idea, a group of recent works \cite{8967840,bs2020integrated,falanga2018pampc,katoch2019edge} seeks to enhance the performance of target identification assisted by the motion strategy. The previous works \cite{8967840,bs2020integrated} of the authors proposed a chasing trajectory to improve the target perception against occlusion from obstacles. Also, the proposed motion strategy includes the travel efficiency and safety jointly. On the other side, the works \cite{falanga2018pampc,katoch2019edge} also aim to enhance the perceptibility of target against the flight motion of camera drone. \cite{falanga2018pampc} leveraged model predictive control (MPC) to account for the motion efficiency and the perception objective so that the projected motion of a target on the image of a drone is minimized. In the case of \cite{katoch2019edge}, they analyzed the motion blur effect from the camera movement to alleviate blurring of the observed edges of a target of interest.
  
\begin{figure}[!t]
\centering
\includegraphics[width=0.4\textwidth]{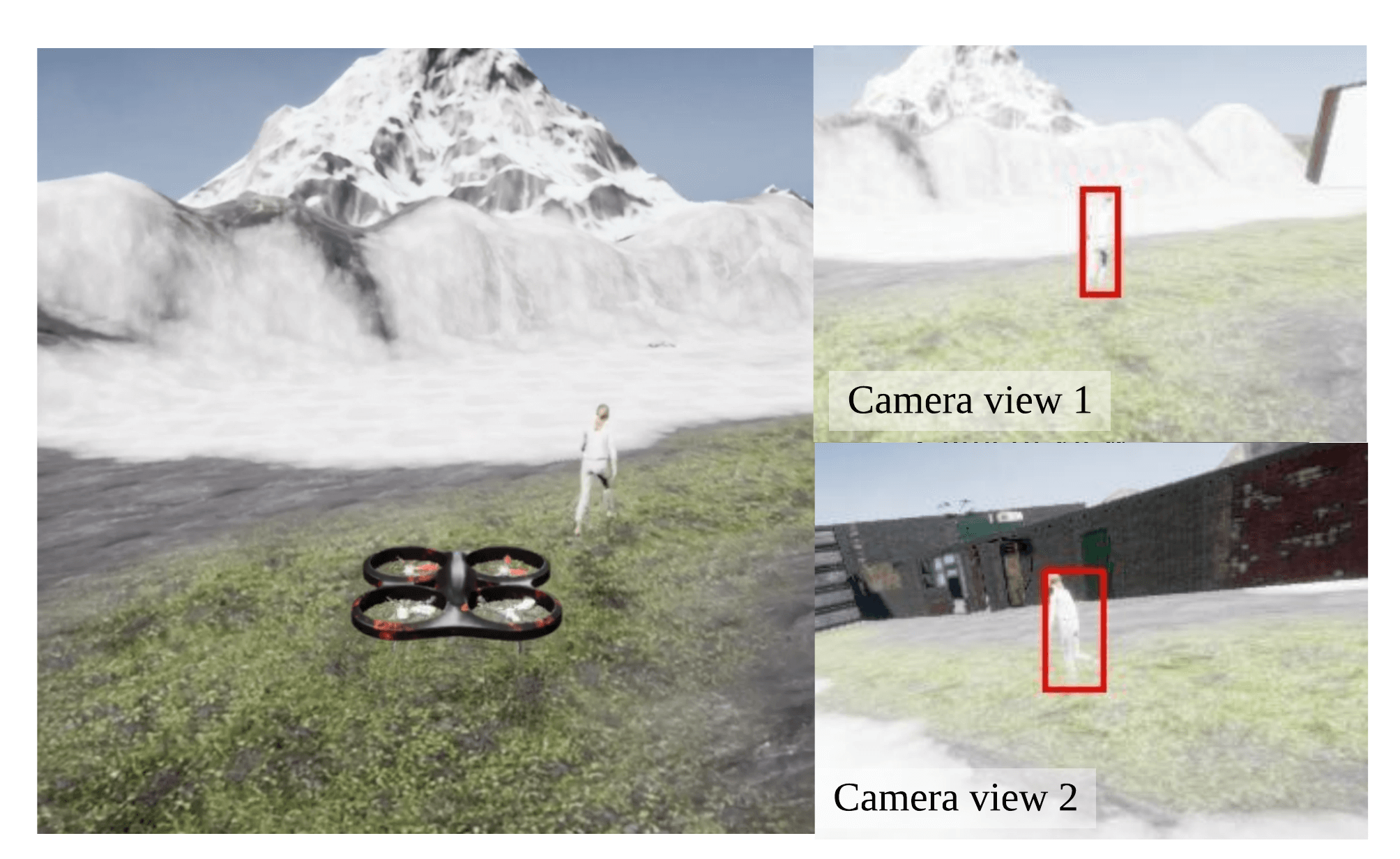}
\DeclareGraphicsExtensions.
\caption{\textbf{Left} : An illustration for autonomous chasing by cinematographer drone. Here, the drone is following the an actor with white clothes in a snow environment. \textbf{Rtight} : Image views from different observation bearing. The position of the actor is same for the two cases. }
\label{Fig_intro}
\end{figure}

\subsection{Target detectability in videography tasks}
% narrowing down 
 By guiding the movement of the camera drone, the mentioned works tried to improve the visual tracking performance focusing on occlusion and motion blur. 
 % importance of detectability in terms of tracker and detector 
  In addition to them, color distinguishability of an actor from the background is also an important factor in localizing the object in the view of a drone. Technically, distinct separability in a color space is advantageous for the detection and tracking methods where the core objective is to build a classifier for an object of interest from backgrounds \cite{henriques2014high,avidan2004support,li2019dsfd}. 
   % importance of detectability in human perception 
  At the same time, detectability is also important for viewers and aesthetic of cinematography (the terms detectability and distinguishability will be interchangably used to refer the color separability of the target and the background). For example, it might be undesirable to constantly keep an actor wearing a white clothes on a white background (compare the two camera views in the right column of \autoref{Fig_intro}).

  \subsection{Contributions}
% The detectability and dicriminiative analysis exist only in image processing or other area
  In the field of image processing, there have been rich studies to analyze and circumvent the color ambiguities \cite{collins2005online,kristan2007probabilistic,ueda2017color}. However, they deal with predetermined video footages, not an active motion strategy to acquire a better-detectable image sequence.
  % bridging the gap between image processing and motion planning
  In this work, we aim to bridge the gap between the research of motion planning of MAVs and resolving color ambiguity in image processing by presenting \textit{detectability aware receding horizon planner} (DA-RHP).  The proposed method generates a chasing path along which a target can be easily distinguishable from environments in terms of color. The contributions of our work can be summarized as the followings:  
 
 \begin{itemize}
    \item We propose a detectability score metric for the motion planner. The metric can quantify how clearly the target can be distinguished from background given a RGB-channeled image.
    \item We build an efficient preplanning process leveraging a graph search method. From construction of a directed acyclic graph to which topological sorting can be determined analytically, we can reduce the computational load in finding an optimal solution.      
    \item We test our algorithm validating the enhanced performance of well-known visual detection and tracking algorithms.
\end{itemize}
 
 To the best knowledge of the authors, there is few studies in motion planning which efforts to generate a trajectory for chasing a dynamic object to improve the detectability using color information of a target and an environment.

\section{Problem statement}
In this section, we put forward the assumptions and the aimed capabilities of DA-RHP.
% the spec of camera drone
Regarding a camera drone, we assume that the allowable velocity is larger than a dynamic object and the drone has a vision sensor which can compute the state of the actor \cite{3dtarget,bonatti2019toward}. For the ease of discussion, we assume the process is reliable enough once the object is successfully identified in the image of the drone. We will denote the pose observation of the target as $T_{a,i}\in SE(3)$ at the $i$th time step. 
% target assumption
 In this work, the future movement of target over a time window is to be predicted from observations. We write them as $\predSeq$ where $N$ is the number of discretization over a horizon.
% pcl of back and target is used as inputs 
Letting a RGB point cloud $\pcl =\{(\transl,\colorTriple)\; |\; \transl \in \RRR , \colorTriple \in \CC\} $ where $\CC$ is a set of RGB-triplets $(r,g,b)$, we assume that the appearance models of the actor and the background are available as $\Paa$ and $\Pb$ respectively. 

\begin{figure}[!t]
\centering
\includegraphics[width=0.35\textwidth]{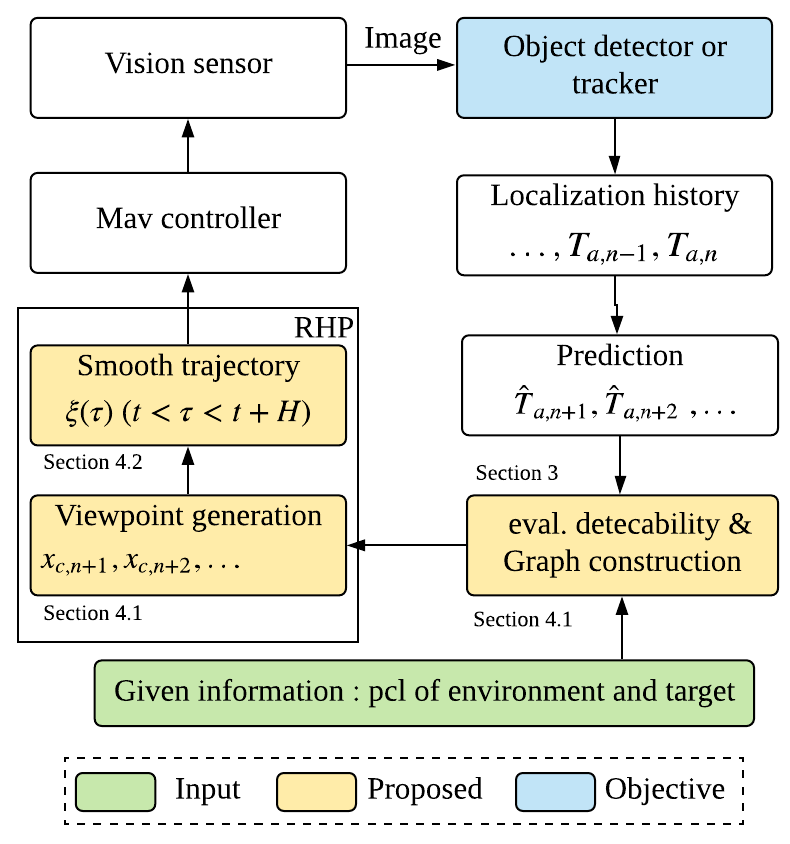}
\DeclareGraphicsExtensions.
\caption{The overall system description for the receding horizon planner. In the pipeline of the general autonomous target following framework, this work tackles the receding horizon planner (RHP) considering detectability given the RGB point cloud of background and target in order to enhance the performance of color-based detection and tracking algorithms.}
\label{fig_diagram}
\end{figure}

% Big picture
On top of the assumptions, we focus on an autonomous chasing framework as visualized in \autoref{fig_diagram}. On the drone side, it localizes and gathers the target observation from the visual information. Based on the observations, we predict the target motion and utilize it as an input of a receding horizon planner (RHP) for a local horizon. Then, RHP outputs a chasing trajectory as a control input for the MAV. The loop is continued until the end of the mission.   
% Objective 
 In the pipeline, we focus on an online chasing strategy which aims at achieving 1) enhanced performance of object detectors and trackers along with human perception, 2) travel efficiency with dynamic feasibility for the drone, and 3) an efficient computation to be used as RHP to rapidly respond the behavior of target. 

\subsection{Outline}
% Method overview 
To achieve the objectives, our planning strategy runs the three modules: detectability-evaluated graph generation, viewpoints optimization and continuous trajectory generation. 
% detectability-evaluated graph
We first consider a set of candidate viewpoints for a prediction $\predI$ and evaluate the target detectability with respect to the viewpoints based on a metric. We introduce the metric in the upcoming section.
% Viewpoint optimization 
Then,  \cref{sec: traj_gen_viewpoints} discusses how to compute a sequence of viewpoints  $\viewSeq$ given the prediction set $\predSeq$, which optimizes the translational distance and the color detectability metric based on a directed acyclic graph search. We will also explore how to reduce the complexity exploiting the structure of the graph.
% contiunous trajectory 
As the last step, a dynamically feasible trajectory is obtained using $\viewPointI$ as a skeleton, which is explained in \cref{sec: traj_gen_smooth}. 

%%%%%%%%%%%%%%%%%%%%%%%%%%%%%%% SECTION 3 %%%%%%%%%%%%%%%%%%%%%%%%%%%%%%%%%%%%%%%%

\section{Evaluation of target Detectability }
\label{sec: detect_metric}
% For one pose
In this section, we describe the evaluation of detectability given a target prediction $\pred$ captured in the image of the camera drone.
% Description
For the purpose, we first quantify the color separability between the foreground and background, given an image which has an object of interest. 
% Related work
In the field of image processing, the work \cite{collins2005online}  addressed this point in order to adaptively select a 1D feature space which best distinguishes between object and background. Based on the log-likelihood ratio of the two distributions on the foreground and background pixels, they measured the \textit{variance ratio} as an extension of Linear Discriminant Analysis (LDA) to incorporate the multi-modality of the 1D features of an image. 

% 3D space
While \cite{collins2005online} proposed a well-suited metric for color detectability of a target, the original work appears to have two shortcomings which necessitate modifications for our scenario. First, recent object trackers and detectors such as \cite{lukezic2017discriminative,rgbd,lan2018rgbI,redmon2016you} exploit multi-channeled feature space to improve accuracy, processing the multi-dimensional feature in real-time supported by the enhanced computing power. For those detection and tracking modules to be applicable to the camera drone, we will evaluate the color separability in the RGB space rather than limiting it into 1D feature space as \cite{collins2005online} did.    

% pixel proximity 
Second, the original method \cite{collins2005online} does not take spatial information of pixels into account when computing detectability, as only the color value of pixel is considered. As a result, the formulation in \cite{collins2005online} might associate high detectability with an image which includes a small part of background area which has similar colors with target even though the pixels are very close to the target. As concrete examples, \cite{collins2005online} gives a higher score to \autoref{fig_detection_comp}-(d) more than \autoref{fig_detection_comp}-(b) in terms of detectability although the overlapped region with background of the similar color is larger than \autoref{fig_detection_comp}-(b). In general, trackers such as \cite{henriques2014high,lukezic2017discriminative} perform template matching starting from the tracking result of the previous step. Thus, encoding the spatial proximity along with the color similarity is justifiable. We will consider not only the color but also the location of pixels in developing the detectability metric.

\subsection{Color detectability computation}

In order to define the detectability of the target projected in the image of the drone, we first consider a transformed pointcloud of the actor $\Paa$ with $\pred\in SE(3)$,  which can be written as
\begin{equation}
 \Paa(\pred) = \{(\pred \transl,\colorTriple)\; |\; (\transl,\colorTriple)\in P_a \}.   
\end{equation}
% image synthesis
Now, we can synthesize a projection image $I_s : \imagePixel \in \RR \rightarrow \colorTriple \in \CC $ from the union of $\Pb$ and $\Paa(\pred)$ given a camera pose and intrinsic parameters in a similar manner with \cite{falanga2018pampc}. In this work, $I_s$ is assumed as an image which has pixels representing both target and background. Examples of rendered images in this way are depicted in (a) and (b) of \autoref{fig_detection_eval}. Naturally, we can identify whether an image coordinate $\imagePixel$ was projected from $\Paa(\pred)$ or $\Pb$. We denote the set of $\imagePixel$ from $\Paa(\pred)$ and $\Pb$ as $\imagePixelTargetSet$, $\imagePixelBackSet \subset \RR$  respectively.        
% detectibility 
Now we are now ready to compute the detectability score of a rendered image $I_s$  extending the metric proposed in \cite{collins2005online}. 

% previous work
    
% modification
\begin{figure}[!t]
\centering
\includegraphics[width=0.48\textwidth]{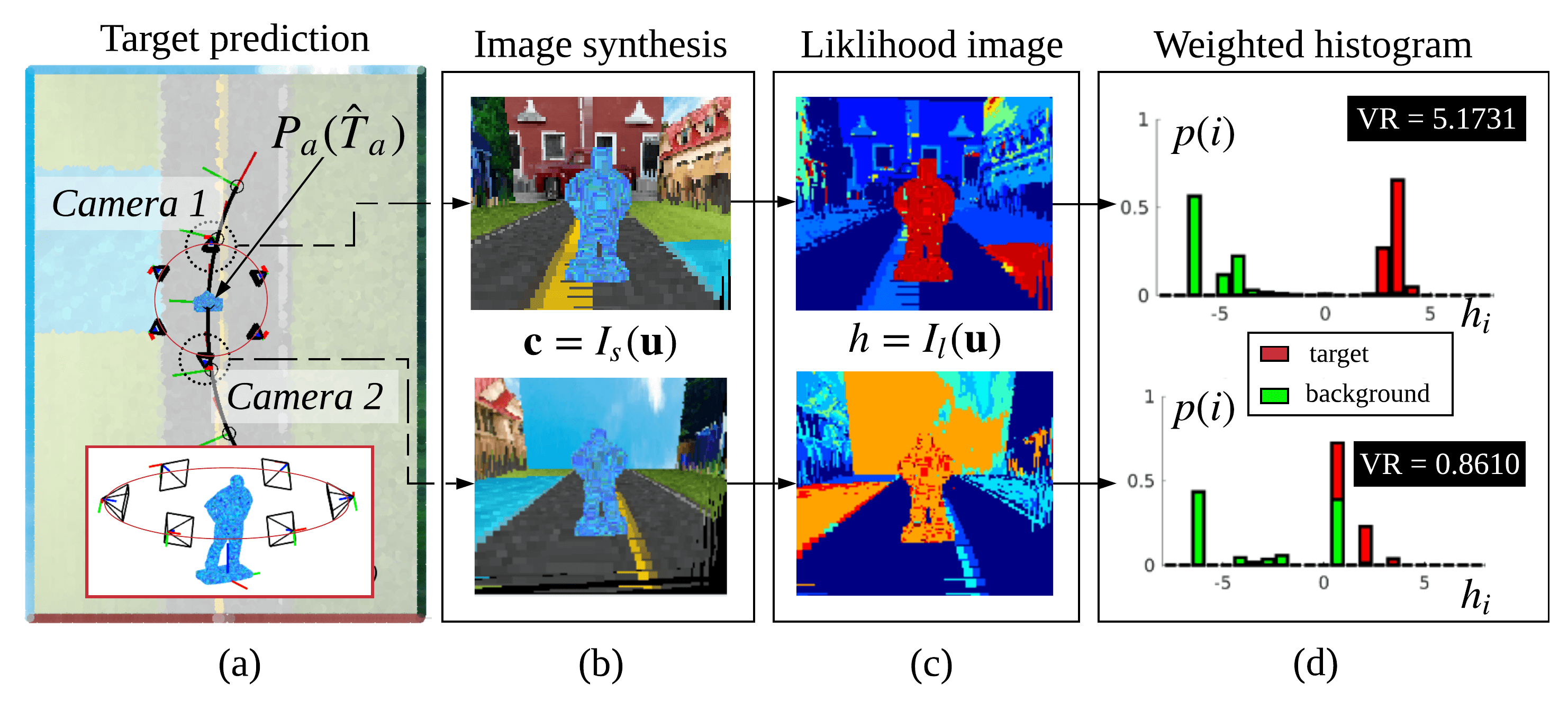}
\DeclareGraphicsExtensions.
\caption{Detectability evaluation process for a prediction pose of a target $\pred$. (a) A sky-blue coloured target $\Paa$ with two different cameras observing it. (b) Rendered RGB images $I_s$ of $\Paa(\pred)$ and $\Pb$ from the two cameras. (c) log-likelihood image $I_l$ for the two images. The mono-scaled image is shown in jet-colormap scale where red denotes high values and blue for low values. (d) Histograms of value in target pixels (red) and background pixels (green) for each $I_l$. It is noticed that the separability  between the two distribution $\histVal_a$ and $\histVal_b$ is more outstanding for the images from \textit{camera 1} than \textit{camera 2}. The values of variance ratio of the two images are stamped in the black boxes.}
\label{fig_detection_eval}
\end{figure}
% modification justification 
% RGB

We now proceed to create a RGB bin for the pixels $\imagePixel $ in a set $\imagePixelSet \in \RR$ of the $I_s$ by assigning a discretized bin to each pixel $\imagePixel \in \imagePixelSet $ based on the RGB-color value $\colorTriple = I_s(\imagePixel)$. The resultant RGB-binning of $\imagePixelSet$ defines a map $\binMap : \CC \rightarrow [0,1]$ which is written as

\begin{equation}
    \binMap(\colorTriple \mid \imagePixelSet) = \dfrac{|Q(\colorTriple \mid \imagePixelSet)|}{|\imagePixelSet|}
    \label{eqn_bin_map}
\end{equation}
where $|\cdot|$ is the cardinality of a set, and $Q(\colorTriple\mid \imagePixelSet)$ is a set of pixels in $\imagePixelSet$ which are grouped into the same bin with a query color $\colorTriple$. The map \eqref{eqn_bin_map} gives the density of the color value $\colorTriple \in \CC $ in the RGB bin created from pixels $U$. 
Now a mono-scaled image $I_l : \RR \rightarrow \R$ is defined as
\begin{equation}
    I_l(\imagePixel) = \log\dfrac{\max(\binMap(\colorTriple\mid\imagePixelTargetSet),\epsilon)}{\max(\binMap(\colorTriple\mid\imagePixelBackSet),\epsilon)},
    \label{eqn_likelihood_def}
\end{equation}
where $\colorTriple = I_s(\imagePixel)$ and $\epsilon$ is a small positive number to avoid numerical issues. Let us call $I_l$ as \textit{log-likelihood ratio image} or \textit{likelihood image} for the brevity. 
 Basically, $\aLikelihood = I_l(\imagePixel)$ converts RGB triplets $\colorTriple = I_s(\imagePixel)$ into a scalar $\aLikelihood$ which encodes a likelihood that a given color $\colorTriple$ belongs to the foreground pixels $\imagePixelTargetSet$ than background pixels $\imagePixelBackSet$. Examples of $I_l$ are illustrated in  \autoref{fig_detection_eval}-(c) for synthesized images in \autoref{fig_detection_eval}-(b). 
 
 We now move on to creation of two histograms (i.e. 1D bin) for $\imagePixelTargetSet$ and $\imagePixelBackSet$ in a likelihood image $I_l $. We write a normalized histogram value $\histVal(i)\in[0,1]$ for a log likelihood value $\aLikelihood_i\in\R$ of the $i$th bin, while representing the histograms of $\imagePixelTargetSet$ and $\imagePixelBackSet$ as $\histVal_{a}(i)$ and $\histVal_{b}(i)$ respectively. Regarding data counting during histogram creation, an unweighted histogram adds an equal amount to the corresponding bin while a \textit{weighted histogram} gives a weight when counting a data. Unlike \cite{collins2005online} where all the histograms $\histVal_{a}(i)$,\; $\histVal_{b}(i)$ were an unweighted ones, we build an unweighted histogram for $\imagePixelTargetSet$ and a weighted histogram for $\imagePixelBackSet$.  In  building a weighted histogram for background, we assign a weight $\liklihoodWeight(\imagePixel)$ to a likelihood value $\aLikelihood = I_l(\imagePixel)$ of $\imagePixel \in \imagePixelBackSet $, given by 
  
 \begin{equation}
 \liklihoodWeight(\imagePixel) =     
 \begin{cases}
\liklihoodWeight_{max}(1-\dfrac{d}{d_c})+\dfrac{d}{d_c}, & d \leq d_c \\
1. & d > d_c \\
\end{cases}
\label{eqn_weighting}
 \end{equation}
 where $d = \|\imagePixel - \overline{\imagePixel}_a \| $ and $\overline{\imagePixel}_a$ is the centroid of the target pixels $\imagePixelTargetSet$. That is, $\liklihoodWeight(\imagePixel)$ linearly increases in proportion to the proximity with the target pixels if $\imagePixel $ is within a boundary $d_c$. \autoref{fig_detection_eval}-(d) visualizes the histograms $\histVal_{a}(i)$ (red) and $\histVal_{b}(i)$ (green) from likelihood images.
 % Definition ...
  As the last phase to define the metric for detectability, we quantify the separability between the two distributions $\histVal_{a}$ and $\histVal_{b}$ by computing a variance ratio $\VR$ of \textit{between-class variance} to the sum of \textit{within-class variances} given by   
 \begin{equation}
     \VR = \dfrac{\mathrm{Var}(\aLikelihood;(\histVal_{a}+\histVal_{b})/2)}{\mathrm{Var}(\aLikelihood;\histVal_{a}) + \mathrm{Var}(\aLikelihood;\histVal_{b})}
     \label{eqn_VR}
 \end{equation}
 where $\mathrm{Var}(h;p)$ is a variance of likelihood $h$ with respect to the probability distribution from a histogram $p$. The metric \eqref{eqn_VR} captures how tightly the  likelihood values $\aLikelihood = I_l(\imagePixel)$ of the target and background pixels are clustered within each group (denominator), and how much the two are spread apart (numerator). As \eqref{eqn_VR} quantifies the separability of the target-likelihood distributions of target pixels and background pixels, we utilize \eqref{eqn_VR} as a score for detectability. The raw values $\VR$ of the two images from \textit{camera 1} and \textit{camera 2} are stamped in \autoref{fig_detection_eval}-(d) in the black-labelled boxes. We can notice that the high-scored camera view in \autoref{fig_detection_eval} (camera 1) has more distinct background color against the target, which is also reflected as the large separation between the histograms of target and background.

\subsection{Discussion of the {proposed metric} }

\begin{figure}[!t]
\centering
\includegraphics[width=0.5\textwidth]{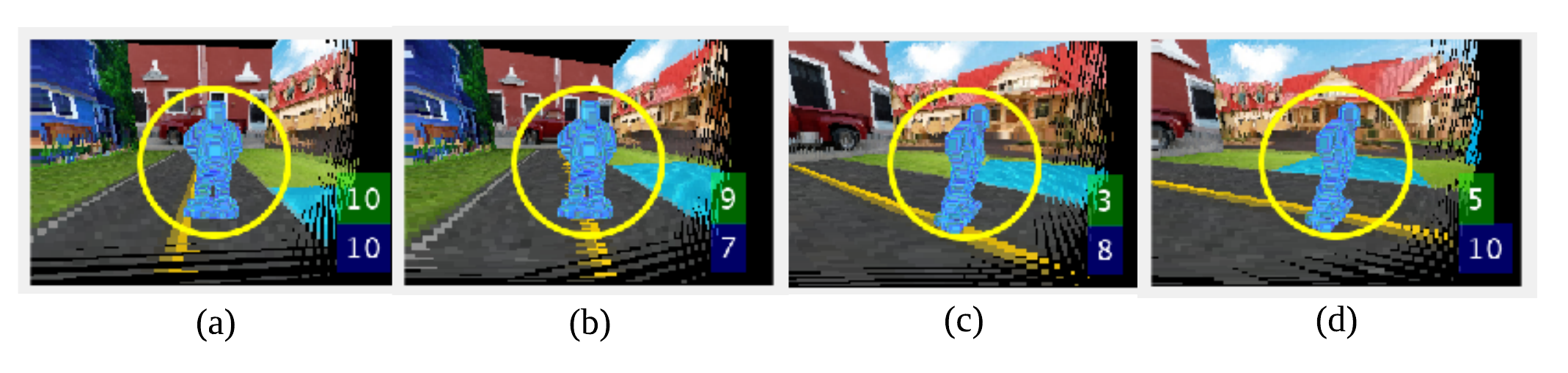}
\DeclareGraphicsExtensions.
\caption{The evaluations of detectability with $\VR$ for several footages of the scenario \autoref{fig_detection_eval}-(a). We linearly scaled the value of $\VR$ to $[1,10]$ for 4 cases (a)-(d) so that (a) has the highest score $10$. The green labels with numbers show the score using weighted histogram proposed in this work while the blue labels shows scores acquired from \cite{collins2005online}. }
\label{fig_detection_comp}
\end{figure} 

In developing the metric for target detectability based on the distributions of pixels of the background and the foreground, {there were two modifications from \cite{collins2005online}.} First, we have utilized the full-color information of pixels by creating a three-dimensional bin when computing the log-likelihood ratio. From this, we were able to define separability in RGB color space which is frequently utilized as the raw input in the state-of-the-art algorithms for detection and tracking \cite{li2019dsfd,redmon2018yolov3,lukezic2017discriminative}. {Second,} we were able to exploit the spatial information of pixels in \eqref{eqn_weighting} by building a weighted histogram for the background pixels. Examples of evaluations of $\VR$ of ours versus \cite{collins2005online} are compared in \autoref{fig_detection_comp}. Both methods give the highest score for \autoref{fig_detection_comp}-(a). However, \cite{collins2005online} gives a high score for \autoref{fig_detection_comp}-(d) as the image has only small amount of dubious colors despite of their close proximity to the target pixels. In contrast, the weighted histogram proposed in this work gives the lower scores to \autoref{fig_detection_comp}-(c) and (d) considering the locations of the pixels of similar colors.          

\section{detection-aware trajectory generation}
\label{sec: traj_gen}
     
\begin{figure}[!t]
\centering
\includegraphics[width=0.47\textwidth]{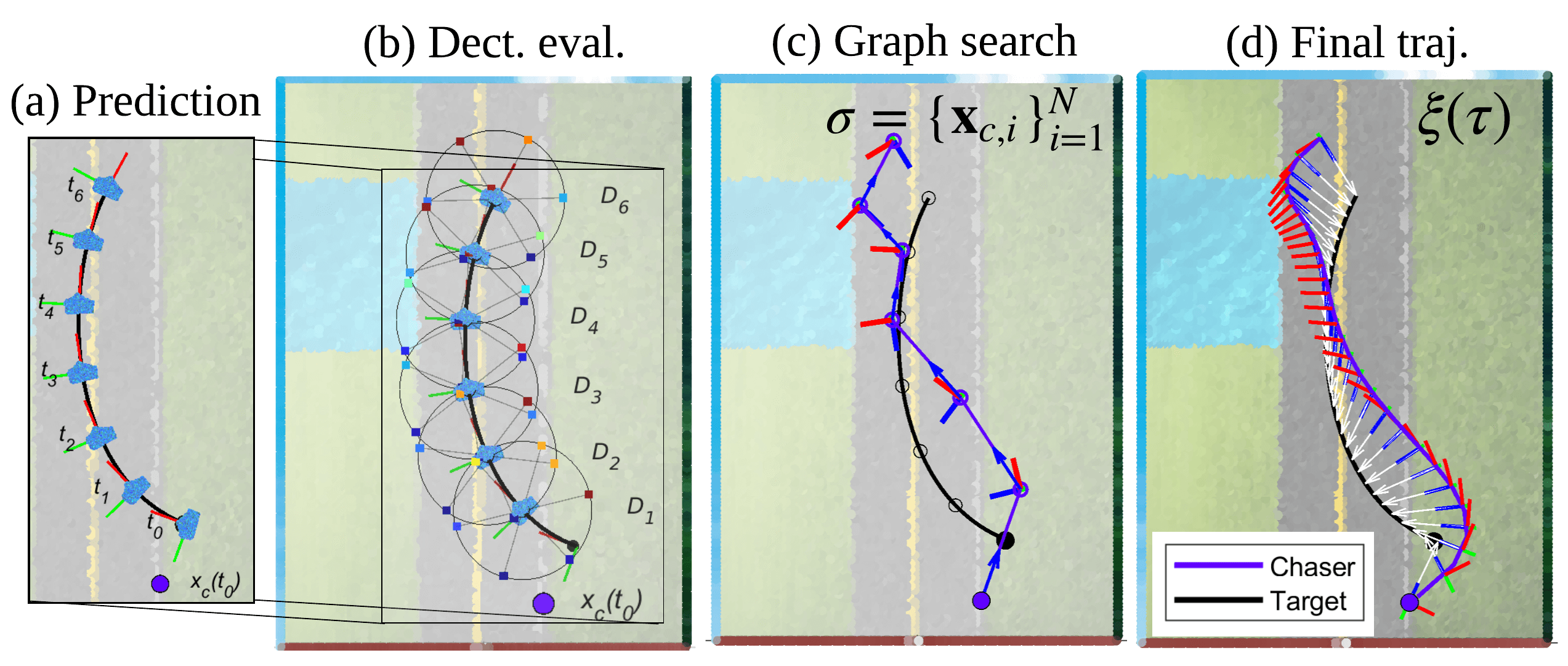}
\DeclareGraphicsExtensions.
\caption{Trajectory generation process. (a) A six-step prediction sequence for the blue-coloured target with the initial location of chaser $\viewPoint_{c,0}$ (purple circle)   (b) For each prediction, the view sphere $D_i$ is composed of six candidate viewpoints. The color of each viewpoint shows the scale of detectability score with jet-colormap. We build a DAG on top of the set of viewpoints having $x_c(t_0)$ as a root. (c) The graph search result for discrete viewpoints at each time step. (d) The final trajectory. The optical axis (z) and x-axis of drone-body coordinate are stamped with bearing vector (white arrow).}
\label{fig_planning}
\end{figure} 

\subsection{Computation of optimal viewpoints}
\label{sec: traj_gen_viewpoints}
In the previous section, we defined a detectability metric as the variance ratio \eqref{eqn_VR} starting from a synthesized image $I_s$ given a camera viewpoint. Making use of the scoring metric $\VR$ and the initial position of chaser $\viewPoint_{c,0} = \viewPoint_c(t_0)\in \RRR $, we explain how to plan discrete view points $\viewPath = \viewSeqRef  \; (\viewPointI \in \RRR)$ for a prediction sequence $\predSeqRef$ along which the accumulative detectability is maximized while the total travel is reduced. From this, we will use $\viewPointI$ as a planned position for the chaser at time $t_i$ with the optical axis toward a target prediction  $\predI = \pred(t_i)$. 

As a requirement of $\viewPath$, we constrain the Euclidean distance of the predicted target and the chaser with $\distRelative \in \R^{+}$. Also, we bound the inter-distance of points $\viewPointI,\viewPointII$ at steps $t_i$ and $t_i + \Delta t$ below $\distMax$, assuming that the velocity larger than $\distRelative/\Delta t$ is undesirable. Coupling the two constraints with the detectability objective, we formulate an optimization as below. 

\begin{equation}
\begin{aligned}
    & \underset{{\viewPath}}{\text{min}}
    &&  \sum_{i=0}^{N-1} \underbrace{\| \viewPointI - \viewPointI{}_{+1}\|}_{\mathrm{distance}} + \lambda\sum_{i=1}^{N} \underbrace{\costDetectI}_{\mathrm{detection}}  \\
    & \text{subject to} & &\text{$\| \viewPointI - \predPointI \| = \distRelative $} \\ 
    & & & \| \viewPointI - \viewPointI{}_{+1} \| \leq \distMax ,    
\end{aligned}
\label{eqn_discrete_opt}
\end{equation}
where $\costDetectI$ is a positive-valued cost function which penalizes a low value of $\VR$ where $I_s$ is the rendered image of $\Paa(\predI)$ by the camera with center  $\viewPointI$ and optical axis $\predPointI - \viewPointI $ (we represented $\predPointI \in \RRR $ as a translation of $\predI$). $\lambda$ is the importance weight for detectability. Instead of searching $\viewPointI$ on the entire half-sphere with a center $\predPointI$ and radius $\distRelative$ to solve \eqref{eqn_discrete_opt}, we compute $\viewPointI$ in a finite set $\viewSphere_i$ which is a  set of discretized points satisfying $ \| \transl - \predPointI \| = \distRelative$ (see the surrounding cameras in \autoref{fig_detection_eval}-(a) as an example). 

Now we consider a directed acylic graph (DAG) $\Graph = (\vertSet,\edgeSet)$ to find an optimal $\viewPath = \viewSeqRef$ on a sequence of the finite set $\viewPointI \in \viewSphere_i$. Setting $\vertSet_0 =\viewSphere_0 =\{\viewPoint_{c,0}\}$, we define the set of vertices $\vertSet$ as 
\begin{equation}
    \vertSet = \bigcup_{i=0}^{N} \vertSet_i  \;,\; \vertSet_i \subset \viewSphere_i
\end{equation}
where $\vertSet_i$ is a subset of $\viewSphere_i$. The graph $\Graph$ is constructed as follows. First, we initialize $\vertSet_i = \viewSphere_i$. Then, we wire two nodes $\viewPointI\in\viewSphere_i$ and $\viewPoint_{c,i-1}\in\viewSphere_{i-1}$ as a directed edge $e= (\viewPoint_{c,i-1},\viewPointI)$ if the condition $\|\viewPoint_{c,i-1} - \viewPointI \|\leq \distMax$ is met. Then, we collect the edge to $\edgeSet$. After all the admissible edges are found, $\vertSet_i$ is reset to have only reachable elements $\viewPointI\in \viewSphere_i$ from the root $\viewPoint_{c,0}$. A structure of $\Graph$ for the case \autoref{fig_planning}-(a) is visualized in \autoref{fig_graph}. We now define an weight $\weight_i$ for $e= (\viewPointI{}_{-1},\viewPointI)$ as 

\begin{equation}
    \weight_i =  \| \viewPointI{}_{-1} - \viewPointI  \| + \lambda\costDetectI.
\end{equation}
Thus, applying a graph-search method for $G$ becomes equivalent to solving \eqref{eqn_discrete_opt}. 
% graph solve ?
Performing a graph-search on a general DAG  is composed of two steps: 1) \textit{topological sorting} and 2) \textit{edge relaxation} \cite{jungnickel2005graphs}. Here, we introduce the following definitions before discussing the topological sorting for the graph $\Graph$.    

\begin{figure}[!t]
\centering
\includegraphics[width=0.4\textwidth]{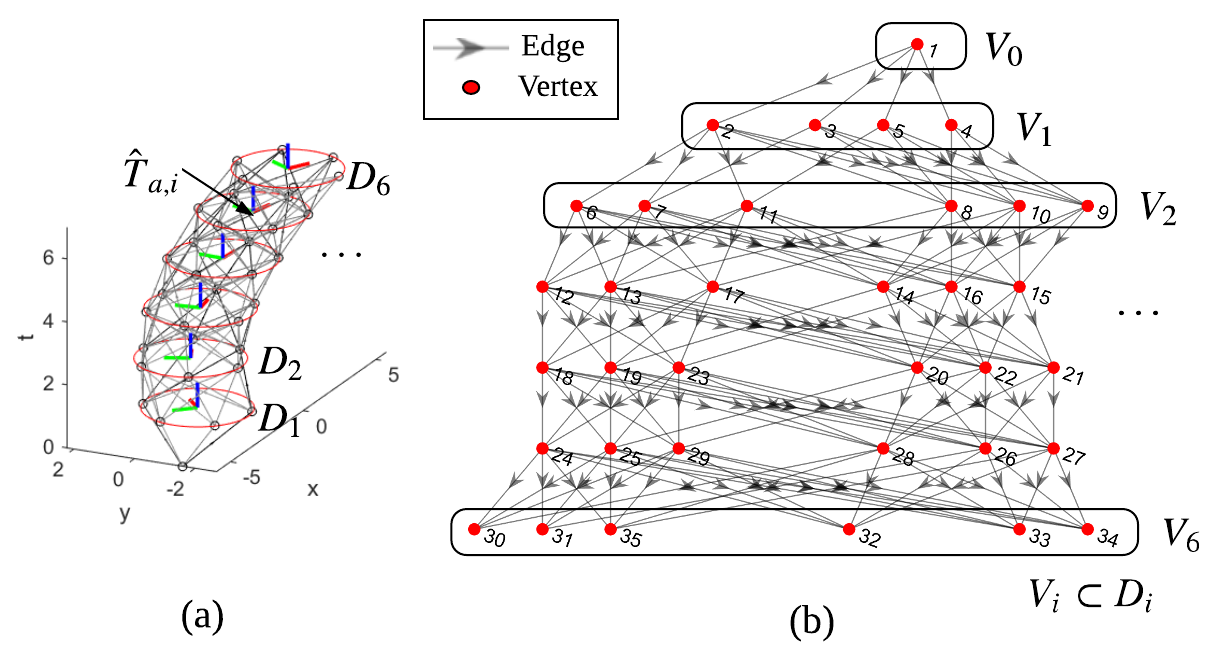}
\DeclareGraphicsExtensions.
\caption{(a) An illustration of graph construction from view spheres of \autoref{fig_planning}-(b). Due to the connection limit $\distMax$, we have two unreachable points in $D_1$ from the root $\viewPoint_{c,0}$ resulting in only 4 vertices in $V_1$. (b) The structure of DAG made from (a). We perform a graph search to find an optimal sequence from the root to all the possible destinations in $V_6$.  }
\label{fig_graph}
\end{figure} 
\begin{definition}
\textit{A sorting $S$ for a set $A$ is defined as a bijective function $S: A \rightarrow \{n|1\leq n \leq |A|\}$.}
\end{definition}
\begin{definition}
\textit{A topological sorting for a directed acyclic graph $G=(V,E)$ is a sorting $S_T: V \rightarrow \{n|1\leq n \leq |V|\}$ satisfying $S_T(u)<S_T(v)$ given any edge $e=(u,v)\in E$ where $u,v\in V$.}  
\end{definition}
The time complexity required for topological sorting in a general DAG is $O(|\vertSet|+|\edgeSet|)$ when using depth-first search DFS with an extra stack \cite{cormen2009introduction}. In the case of $\Graph$, however, we can directly determine a topological sorting $S_T$. For the enumeration purpose,  we represent $\vertSet_i$ as $ \{\viewPointI^{j}\}_{j=1}^{n_i}$ where $n_i$ is cardinality of $\vertSet_i$. We determine the sorted index for a vertex  $\viewPointI^{j}$ as below:

 \begin{equation}
 S_T (\viewPointI^{j}) =     
 \begin{cases}
\sum_{k=0}^{i-1}n_k +j & i>0, \\
1 & i = 0. \\
\end{cases}
\label{eqn_sorting}
 \end{equation}
 Now we show \eqref{eqn_sorting} is a topological sorting by \textbf{Proposition 1}.
\begin{proposition}
\textit{The sorting of \eqref{eqn_sorting} is a topological sorting for $\Graph$.}
\label{propose1}
\end{proposition}
\begin{proof}
Following the wiring process mentioned above, every edge $e\in \edgeSet$ can be written as $e = (\viewPointI^{j_1},\viewPointII^{j_2})$ where $\viewPointI^{j_1} \in \vertSet_i$, $\viewPointII^{j_2} \in \vertSet_{i+1}$ and $1\leq j_1 \leq n_i,\; 1\leq j_2 \leq n_{i+1}$. 

From this, $\sum_{k=0}^{i-1}n_k < S_T (\viewPointI^{j_1}) \leq \sum_{k=0}^{i}n_k$ and $\sum_{k=0}^{i}n_k < S_T (\viewPointI^{j_2}) \leq \sum_{k=0}^{i+1}n_k$ hold. As $S_T (\viewPointI^{j_1})<S_T (\viewPointI^{j_2})$ is satisfied for every edge, $S_T$ is  a topological sorting for $\Graph$.          
\end{proof}
This decreases the computation time by sparing us the need for an extra algorithm for topological sorting. Based on the vertex enumeration from \eqref{eqn_sorting}, a dynamic programming \cite{cormen2009introduction} can compute the shortest path to every node of $\Graph$ from the root node using only edge relaxation step whose complexity is linear to the number of edges.  Thus, a discrete path $\viewPath = \viewSeqRef$ which optimizes \eqref{eqn_discrete_opt} can be found by identifying the minimum cost among all the pairs ($\viewPoint_{c,0}$,$\viewPoint_{c,N}^{j}$) from the root to vertices in $\vertSet_N$. An example of $\Graph$ is visualized in \autoref{fig_planning}-(c), which was generated from a prediction $\predSeqRef$ and corresponding view sphere $\{\viewSphere_i\}_{i=1}^{N}$ of \autoref{fig_planning}-(a) and (b). To examine how the deterministic sorting $\eqref{eqn_sorting}$ can reduce the computation time for the overall graph search process for $\Graph$, we compared our implementation with a Matlab built-in function which performs the topological sorting equipped with additional search algorithm. The result is as shown in \autoref{fig_graph_comp}. 

\begin{figure}[!h]
\centering
\includegraphics[width=0.5\textwidth]{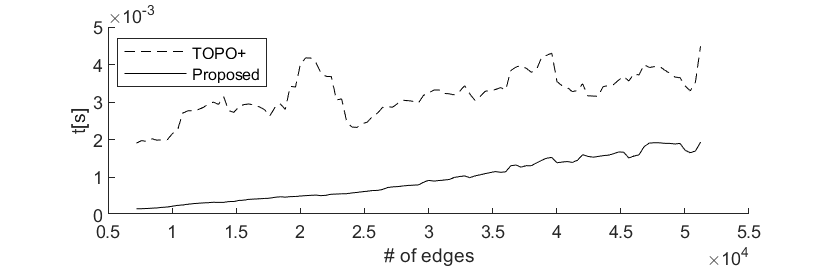}
\DeclareGraphicsExtensions.
\caption{{Computation time to solve \textit{the single-source-shortest-path (SSSP)} for the graphs of the same structure with $\Graph$ by increasing the number of edges. The dashed line denotes the result of applying MATLAB \texttt{shortestpathtree} function to our the graphs, which performs topological sorting using a search-based method. The thick black line shows the result when the sorting process was replaced with \eqref{eqn_sorting}.}}
\label{fig_graph_comp}
\end{figure}

\subsection{Dynamically feasible trajectory for drones}
\label{sec: traj_gen_smooth}

The formulation \eqref{eqn_discrete_opt} and the graph-search have planned a detectability enhanced viewpoint $\viewPointI$ at every time step $t_i$, assuming that target's pose will be $\predI$. Based on the initial sketch $\viewPath = \viewSeqRef$ and its initial state $\viewPointInit,\; \viewVelInit$ and $\viewAccelInit$, we will finalize a continuous trajectory consisting of position and yaw $ \mathbf{\xi}(\T) = [x(\T)\; y(\T)\; z(\T)\; \psi (\T)]^{T} \in \RRRR $ of the camera drone for $t_0 \leq \T \leq t_N $. For the position $\transl_{c}(\tau) = [x(\T)\; y(\T)\; z(\T)]^{T} \in \RRR$, we represent the trajectory with a polynomial spline curve

\begin{equation}
\viewPoint_c(\T) =  
\begin{cases}
 \sum_{k=0}^{\polyOrder}\polyCoeff_{1,k}\tau^k & (t_{0} \leq \tau < t_{1}) \\
 \sum_{k=0}^{\polyOrder}\polyCoeff_{2,k}\tau^k & (t_{1} \leq \tau < t_{2}) \\
...&\\
 \sum_{k=0}^{\polyOrder}\polyCoeff_{N,k}\tau^k & (t_{N-1} \leq \tau < t_{N}) \\
\end{cases}
\end{equation}
where $\polyCoeff_{i,k} \in \RRR $ denotes the polynomial coefficient for order $k$ defined over time segment $[t_i,t_{i+1})$. We want to calculate the polynomials which effort to pass through the viewpoints $\viewPointI$ at time $t_i$ while decreasing the magnitude of the high-order derivative for smooth transition. For the purposes, an optimization can be formulated as
\begin{equation}
\label{eqn_cont opti}
\begin{aligned}
& \underset{}{\text{min}}
&& \int_{t_{0}}^{t_{N}} {\lVert{{\transl_c}^{(3)}(\tau)}\rVert}^2 d\tau \:+\: \rho \sum_{i=1}^{N}{\lVert{\transl_c(t_{i})}-\viewPointI\rVert}^2 \\ 
& \text{subject to} & &  \transl_c(t_{0}) = \viewPointInit  \\
& & &   \dot{\transl}_c(t_{0}) = \viewVelInit\\
& & &   \ddot{\transl}_c(t_{0}) = \viewAccelInit
\end{aligned}
\end{equation}

This can be converted to a quadratic programming with respect to $\polyCoeff_{i,k}$ in a similar way with \cite{chen2016tracking}, which can be solved efficiently using the algorithms such as interior point  \cite{mehrotra1992implementation}. \autoref{fig_planning}-(d) demonstrates a smooth curve obtained from the viewpoints in \autoref{fig_planning}-(c). Regarding the trajectory of $\yaw(\T) \in \R$, we decide it by heading the optical axis of drone to the actual center of the actor in a myopic way. Based on the dynamics and the differential-flatness of a quadrotor model, the trajectory $ \mathbf{\xi}(\T) = [x(\T)\; y(\T)\; z(\T)\; \psi (\T)]^{T}$ can be executable with dynamic-feasibility for MAVs within the actuation limits \cite{mellinger2011minimum}. $\mathbf{\xi}(\T)$ is fed into the controller of the drone until a prediction is reliable or the defined period expires (see the final input to MAV controller in \autoref{fig_diagram}). 

 Up to now, we have explained the two steps to generate a detection-aware chasing trajectory: 1) detectability-optimized viewpoints generation and 2) smooth trajectory generation. We now put them together to build DA-RHP as summarized in \textbf{Algorithm 1}. 

\begin{algorithm}
\DontPrintSemicolon
\SetAlgoLined
\SetKwInOut{Input}{Input}\SetKwInOut{Output}{Output}
\SetKwInOut{given}{Input}
\SetKwInOut{param}{Parameter}
\SetKwInOut{initialize}{Initialize}
\SetKwFunction{DAGSearch}{DAGSearch}
\SetKwFunction{smoothTraj}{smoothTraj}
\SetKwFunction{detectabilityEval}{detectabilityEval}
\SetKwFunction{predict}{Predict}
\SetKwFunction{viewsphere}{ViewSphere}
\given{Pointclouds $\Paa$, $\Pb$. Mission time $[t_s,t_f]$}
\initialize{Accumulated prediction error; $\texttt{accumErr}=0$}
\For{$t=t_s$  \KwTo $t_f$}{ 
        \If{$\texttt{accumErr}>\epsilon$}{
            Set window $\tau\in[t,t+H]$, discretization $\{t_n,t_{n+1},...,t_{n+N}\}$ and  $\transl_{c,n}=\transl_{c}(t)$. \\
            $\predSeq$ = \predict{$\{t_i\}_{i=n+1}^{n+N}$}\\ 
            Eval. detect. for \viewsphere{$\predI$}\\ 
            $\{\viewPoint_{c,n+1},\viewPoint_{c,n+2},...,\viewPoint_{c,n+N}\}$=\DAGSearch() \\
            $\mathbf{\xi}(\T)$ = \smoothTraj($\{\viewPoint_{c,i}\}_{i=n}^{n+N}$) 
            $\texttt{accumErr}=0$            
        }
        update $\texttt{accumErr}$ \\ 
        exec. $\mathbf{\xi}(t)$
}
\caption{DA-RHP}
\end{algorithm}

\section{Validations}
In this section, we validate the proposed planner DA-RHP by comparing it with a plain chasing strategy in a simulated environment. Several key results are presented including computation time, travel distance and the performance of visual identification of a target when using correlation filter-based trackers \cite{henriques2014high,lukezic2017discriminative} and learning-based object detection methods \cite{redmon2018yolov3,li2019dsfd} for the output video footage taken from the drone. 
\subsection{Experimental setup}
% Method 
We performed high-fidelity simulations in \textit{Unreal engine} with Airsim plugin \cite{shah2018airsim} where a drone with vision sensors is supposed to  autonomously follow a walking actor with white clothes. In the environment, there are multiple piles of snow which can interrupt the visual tracking of the actor as shown in \autoref{Fig_intro}. The actor is set to walk with the maximum speed 1.3 \si{\metre}/\si{\second} along the path shown as ruby curve in the first column of \autoref{fig_simulation}. We use constant velocity model for prediction over a short prediction horizon. The drone was mounted with a firmly attached camera having 120\si{\degree} field-of-view.   The whole algorithm in \textbf{Algorithm 1} was written in C++, and ROS was used to operate the simulated drone via px4 SITL. Additionally, we applied QPOASES to compute \eqref{eqn_cont opti}. All the computations including running the simulator were performed in a computer with CPU Intel i7-6700K CPU @ 4.00GHz and RAM 16GB. 

\begin{figure*}[t]
 \centering
  \includegraphics[width=0.95\textwidth]{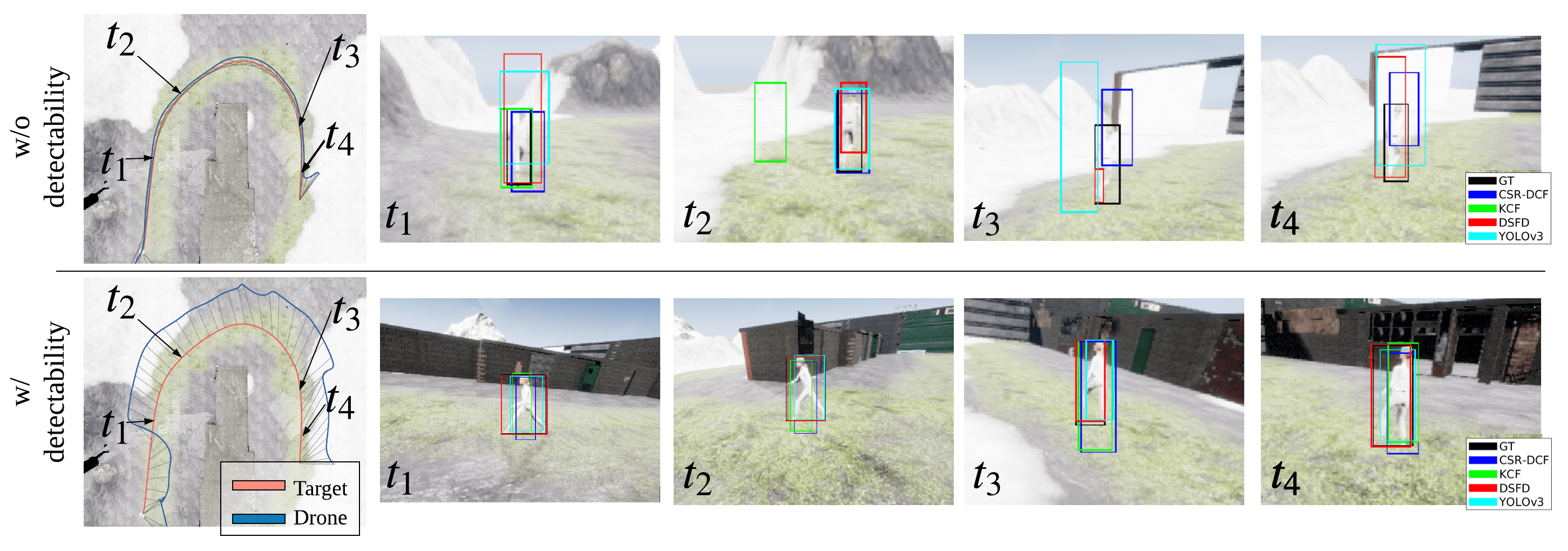}
  \caption{The positional histories (left) and video footage ($t_1 ,..,t_4$) of the two motion strategies. In the image sequence, the detection boxes for the actor are visualized for $4$ different algorithm plus the ground truth.  \textbf{Top}: A plain-chasing strategy. \textbf{Bottom}: The proposed DA-RHP.    }
\label{fig_simulation}
\end{figure*}

\begin{figure}[h]
 \centering
  \includegraphics[width=0.45\textwidth]{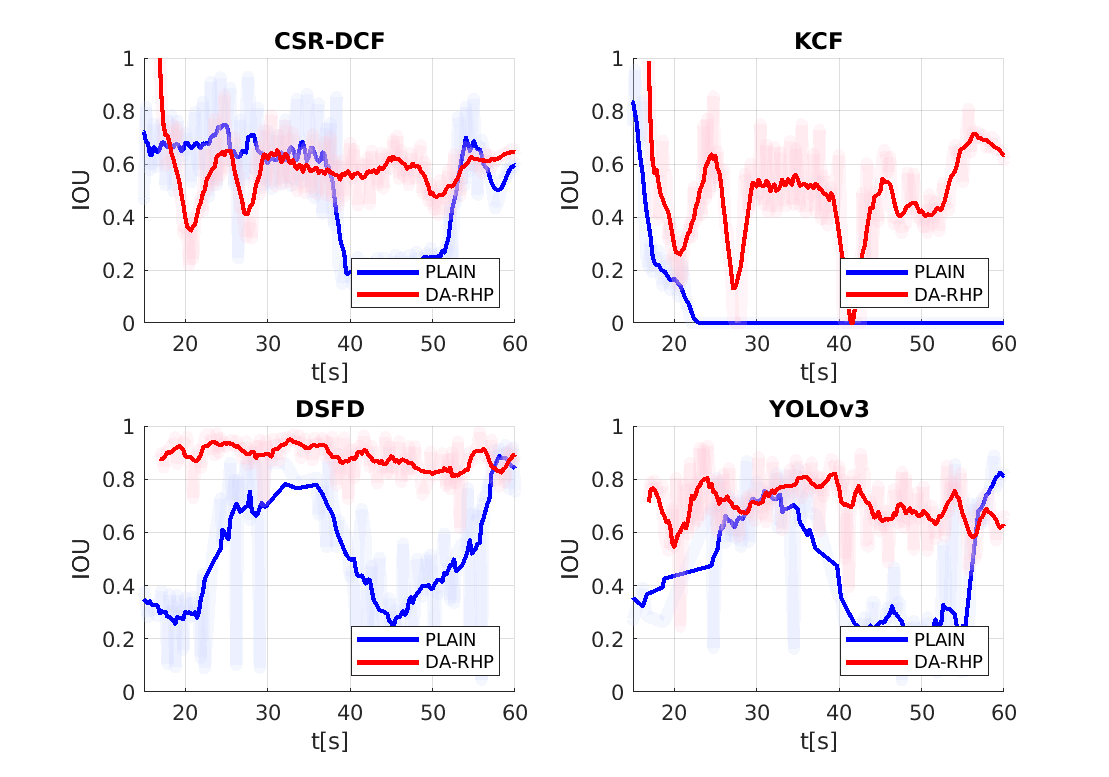}
  \caption{{The IOU (intersection over union) results from the simulation with two tracking algorithms (CSR-DCF, KCF) and two detection algorithms (DSFD, YOLOv3). For each algorithm, the IOU is computed for the footage taken along the trajectory generated from either the proposed (red) or plain receding horizon planner (blue).} }
\label{fig_iou_tracker}
\end{figure}

\begin{figure}[h]
 \centering
  \includegraphics[width=0.45\textwidth]{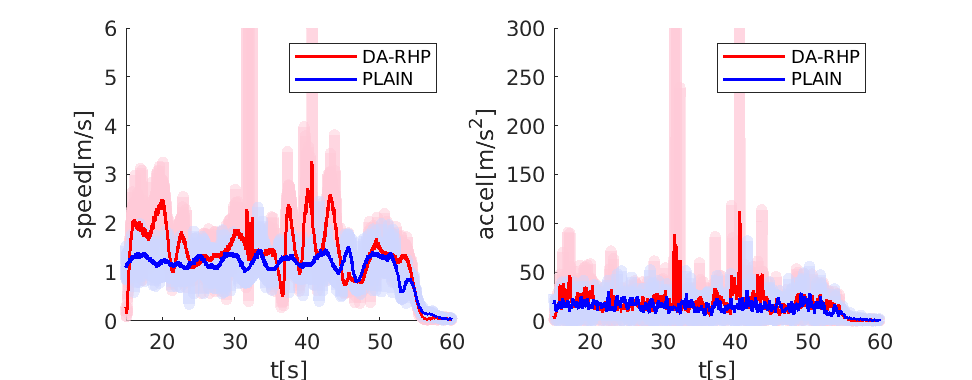}
  \caption{(a) The histories of speed of the drone during the simulations when the drone executes DA-RHP (red) and plain-RHP (blue). (b) Norm of acceleration  results in the simulation.}
\label{fig_vel_accel}
\end{figure}

% simulation description 
Based on them, we perform two simulations. One is a plain chasing planner without detectability consideration and the other is DA-RHP. The only difference is that the former sets the detection importance $\lambda$ to zero while the latter to $20$ in \eqref{eqn_discrete_opt}. Other than that, the same parameters were applied: the time horizon $H=4$ and time discretization $N=4$ were used, while the number of elements in the view-sphere was set to $8$. For the smooth trajectory, the order of polynomial spline was set to $K=5$. The observation distance was chosen as $\distRelative = 5$ \si{\meter}. Around $200,000$ points were included in $\Pb$ for the environment in \autoref{fig_simulation}, and the total duration of the mission is $60$ \si{\second}. 
From the two simulations, we compare the translation histories and the sequences of images from the drone (see \autoref{fig_simulation}). 
 In the simulations,  we fed the current pose information of the target to the follower drone rather than calculating the target's pose from the images. This was to observe how the receding horizon planners operate for the entire duration, not to be disturbed the tracking failure. In fact, without such setup, the planner without detectability consideration will malfunction due to tracking failures. 

\subsection{Results}

% Color
First, we assess the performance of target detectability of DA-RHP based on the two object detection (DSFD \cite{li2019dsfd}, YOLOv3 \cite{redmon2018yolov3})  and two tracking algorithms (CSR-DCF \cite{lukezic2017discriminative}, KCF \cite{henriques2014high}).   
% Intro
To measure the detection accuracy for the four algorithms, the history of IOU (intersection over union) is analyzed as shown in \autoref{fig_iou_tracker} and \autoref{Table:params}. Additionally, the average precision (AP) of the neural networks is presented in the case of detection algorithms as analyzed in the literature such as \cite{redmon2018yolov3,li2019dsfd}.

% Tracker 
For both trackers, the tracking was improved when the target was filmed from DA-RHP, resulting in longer duration with reliable IOU ($\geq 0.4$) than the plain RHP. In the case of the plain RHP, KCF tracker and CSR-DCF started to lose the accuracy from $t_2$ and $t_3$ respectively as shown in \autoref{fig_simulation} when they encountered the snow background. In contrast, DA-RHP took a detour to maintain the actor in front of the brick walls avoiding snow backgrounds. The averages of IOU are summarized in \autoref{Table:params}. 

% Detection 
To validate our algorithm with detection algorithms YOLOv3 and DSFD, we applied the deep neural networks  originally designed for multi-object classification to a single-class setting for our target detection scenario. With supervised learning, each network was trained with $500$ RGB images taken from the drone observing the actor at various locations and bearing direction. This was intended to to mimic a human perception test where a human subject is shown many figures of the actor and requested to segment the actor from the two different video footage from the cinematographer drone. The training set also included highly-ambiguous footages. Then, we tested the footages from the drone during the simulations. To record IOU, we included only the images throttled by the network output confidence $0.5$. The results show the averaged IOU of DA-RHP was higher than the plain strategy as noted in \autoref{Table:params}. 

% other things to be analyzed
Regarding the length of the chasing path, DA-RHP traveled 18 \si{\meter} longer than the plain chasing planner while giving more bearing views advantageous for detectability performance. 
The average computation time was $220$ \si{\milli\second}, $1$ \si{\milli\second} and $3$ \si{\milli\second} for detectability evaluation, graph-search and quadratic programming respectively. The total pipeline ran at $4$-$5$ \si{\Hz} showing the real-time performance to be used as a receding horizon planner. 

\subsection{Discussion}
As mentioned above, most of the computation time was spent on the image synthesis process with respect to all the camera poses in view-sphere. In our implementation, image projection of the point cloud $\Pb$ handled self-occlusion of the objects (e.g. a brick house in the center \autoref{fig_simulation}) in the environment to deal with more general situations. In a simpler scenario without such objects, we can reduce the computation time by omitting the occlusion-culling. It is also noteworthy that DA-RHP showed a slightly degraded detection and tracking performance around $20$ \si{\second} and $40$ \si{\second} as can be seen in \autoref{fig_iou_tracker}. These are associated with an increased velocity and acceleration as visualized in \autoref{fig_vel_accel}. We found that the motion guidance to obtain a high-detectability had caused a perceptible change of the location  and the scale of the target in the image view along with motion blur, resulting in low accuracy during a short period.  In this case, increasing the number of discretization of view spheres or reducing $\distMax$ can mitigate abrupt motions of the drone. On the hardware side, utilizing a sensor with wider FOV or gimbal stabilization can be an option. Also, we found that the nature of DA-RHP to observe more distinguishable background  often changes the relative bearing (see the bottom of \autoref{fig_simulation}), which might make tracking of the target challenging.  
For example, the case with the CSR-DCF tracker (top left of \autoref{fig_iou_tracker}) indicates the temporarily poorer  IOU performance of DA-RHP  than the plain method, up to 
around 40 \si{\sec} ($t_3$ in \autoref{fig_simulation}) 
when the snow misled the tracker. However, it is because the relative view angle of the plain method remained almost same in the top-left figures of \autoref{fig_simulation} during that interval.

\setlength\doublerulesep{0.4pt} 
\begin{table}[h]
\begin{center}
 \begin{adjustbox}{max width=0.48\textwidth}
\begin{tabular}{l|cc|cc|c}
\toprule[1pt]
\hline
\multirow{2}{*}{} & \multicolumn{2}{c|}{Tracking (IOU)} & \multicolumn{2}{c|}{Detection (IOU/AP)} & \multirow{2}{*}{dist. [\si{\meter}]}  \\
                  &     KCF \cite{henriques2014high}      & CSR-DCF \cite{lukezic2017discriminative}           &      DSFD \cite{li2019dsfd}     &      YOLO \cite{redmon2018yolov3}     &  \\\cline{1-6} \hline
        plain RHP          & 0.0522           &          0.5175 &    0.5332/  0.0.3197     &      0.5089 / 0.1879    & 73 \\

          DA-RHP        &    0.4663       &         0.5725  &   0.8876/ 0.7313        &  0.7042 / 0.3277          & 91 \\ 
\hline
\end{tabular}

\end{adjustbox}
\end{center}
 \caption{Performance of detectability }
 \label{Table:params}

\end{table}

\section{Conclusion and future works}
In this work, we presented a detectability metric and a chasing motion strategy named DA-RHP which jointly embraces the color distinguishability and total travel distance. In almost-real simulations, we validated the enhanced performance of the object identifiers by comparing DA-RHP with a plain chasing planner without detectability consideration. We also confirmed the reduced complexity in the graph search method from analytical topological sorting, validating the online performance. 
As future works, 
% 1.Detectability 
we will extend the proposed algorithm into real-world scenarios, validating the color detectability metric in various datasets. 
% 2. collision safety 
Especially, the collision safety will be included by applying the approach proposed in the previous works of the authors \cite{jeon2019online,bs2020integrated}. 
 Also, we will consider cases where the camera drone receives the point cloud information on-the-fly, enabling the capability to handle dynamic environments.               


%%%%%%%%%%%%%%%%%%%%%%%%%%%%%%%%%%%%%%%%%%%%%%%%%%%%%%%%%%%%%%%%%%%%%%%%%%%%%%%%

%%%%%%%%%%%%%%%%%%%%%%%%%%%%%%%%%%%%%%%%%%%%%%%%%%%%%%%%%%%%%%%%%%%%%%%%%%%%%%%%

%%%%%%%%%%%%%%%%%%%%%%%%%%%%%%%%%%%%%%%%%%%%%%%%%%%%%%%%%%%%%%%%%%%%%%%%%%%%%%%%

\bibliographystyle{ieeetr}
\bibliography{bibliography}

\end{document}